%% file: NSGA-III_GECCO.tex
\newif\ifarxiv 
\arxivtrue

\newif\ifreview 
\reviewfalse

\newif\iffinal 
\finaltrue

\ifarxiv
\documentclass{article}
\else
\ifreview
\documentclass[dvipsnames,format=sigconf,anonymous=true,review=true]{acmart}
\else
\documentclass[dvipsnames,format=sigconf]{acmart}
\fi
\fi

\ifarxiv
\else

\copyrightyear{2024}
\acmYear{2024}
\setcopyright{rightsretained}
\acmConference[GECCO '24]{Genetic and Evolutionary Computation Conference}{July 14--18, 2024}{Melbourne, VIC, Australia} \acmBooktitle{Genetic and Evolutionary Computation Conference (GECCO '24), July 14--18, 2024, Melbourne, VIC, Australia} \acmDOI{10.1145/3638529.3654218} \acmISBN{979-8-4007-0494-9/24/07}

\fi

\input{preamble}
\ifarxiv
\else
\fi

\ifarxiv
\usepackage{authblk}
\fi

\begin{document}

\ifarxiv
\author[1]{Andre Opris}
\author[1]{Duc-Cuong Dang}
\author[2]{Frank Neumann}
\author[1]{Dirk Sudholt}
\affil[1]{University of Passau, Passau, Germany}
\affil[2]{The University of Adelaide, Adelaide, Australia}
\date{}
\fi

\title{Runtime Analyses of NSGA-III on Many-Objective Problems}
    
\ifarxiv
\maketitle
\fi
    
\begin{abstract}
NSGA-II and NSGA-III are two of the most popular evolutionary multi-objective algorithms used in practice. While NSGA-II is used for few  objectives such as 2 and~3, NSGA-III is designed to deal with a larger number of objectives. 
In a recent breakthrough, Wietheger and Doerr (IJCAI 2023) gave the first runtime analysis for NSGA-III on the 3-objective \textsc{OneMinMax} problem, showing that this state-of-the-art algorithm can be analyzed rigorously.

We advance this new line of research by presenting the first runtime analyses of NSGA-III on the popular many-objective benchmark problems 
\mLOTZ, \mOMM, and \mCOCZ, for an arbitrary \neweditx{constant} number~$m$ of objectives. 
Our analysis provides ways to set the important parameters of the algorithm: the number of reference points and the population size, so that a good performance can be guaranteed. We show how these parameters should be scaled with the problem dimension, the number of objectives and the fitness range.
%
To our knowledge, these are the first runtime analyses for NSGA-III for more than 3 objectives. 

\end{abstract}



\section{Introduction}\label{sec:intro}

Evolutionary multi-objective (EMO) algorithms 
\neweditx{frequently have been 
applied for 
a variety of} multi-objective optimization problems~\cite{kdeb01,coello2013evolutionary}. 
Those problems have a wide range of applications~\cite{coello2013evolutionary}. 
EMO algorithms, including the most popular ones like NSGA-II~\cite{Deb2002} and NSGA-III~\cite{DebJain2014}, are particularly well suited for dealing with \neweditx{multiple objectives} because they evolve populations of search points. These populations are 
an ideal vehicle to not only compute one single optimal solution but a set of solutions representing the trade-offs of the underlying objective functions.

Runtime analysis~\cite{Jansen2013,DoerrN20} has contributed to the theoretical understanding of evolutionary multi-objective algorithms over the last 20 years. Different directions have been followed during this time. One major direction involved solving single-objective optimization problems through adding additional objectives. This has become known as \emph{multi-objectivization}~\cite{DBLP:conf/emo/KnowlesWC01} or, when solving constrained single-objective problems through relaxing the constraints into additional objectives, under the term \emph{Pareto optimization}.
Pareto optimization for solving constrained single-objective problems \neweditx{establishes} the constraints as an additional objective which allows evolutionary multi-objective algorithms to mimic a classical greedy behavior~\cite{DBLP:journals/nc/NeumannW06,DBLP:journals/ec/FriedrichHHNW10,NeumannWittBook}. This allows to obtain results \neweditx{for a wide range} of settings in the context of submodular optimization~\cite{DBLP:books/sp/ZhouYQ19}.
\neweditx{On the other hand}, the behavior of evolutionary algorithms has been studied for classical multi-objective optimization problems such as the multi-objective variants of minimum spanning trees~\cite{DBLP:journals/eor/Neumann07} and shortest paths~\cite{DBLP:conf/foga/Horoba09}. Furthermore, the impact of different diversity mechanisms such as the density estimate of SPEA2 and the epsilon dominance approach have been analyzed~\cite{DBLP:series/sci/HorobaN10}. Here, the focus has been mainly on their ability to obtain  good approximations for artificial bi-objective benchmark functions with exponentially large Pareto fronts~\cite{DBLP:conf/gecco/HorobaN08,DBLP:conf/foga/HorobaN09}.

Most of the previously mentioned studies consider variants of the algorithm GSEMO
that use built-in mechanisms such as the density estimator in order to point out the benefits or drawbacks of the different components of the algorithm.
Recently, there has been a growing interest in studying the behavior of algorithms such as NSGA-II~\cite{Doerr2023,Zheng2022,Bian2022PPSN}, \nsgaIII~\cite{WiethegerD23}, and MOEA/D~\cite{Huang2021,Huang2019,Huang20211} with respect to their runtime behavior. Results have mainly been obtained for artificial benchmark functions, although there are results for NSGA-II \newedit{on} the multi-objective minimum spanning tree problem~\cite{DBLP:conf/ijcai/CerfDHKW23}, and \neweditx{for} MOEA/D for the more general setting of computing multi-objective minimum weight bases~\cite{do2023rigorous} \neweditx{that hold} for any number $m$ of objectives.

Algorithms such as MOEA/D and \nsgaIII have been developed to deal with problems having 3 or more objectives as NSGA-II, SPEA2 and similar algorithms based mainly on Pareto dominance do not perform well for many objectives due to the increasing number of objective vectors that become incomparable. This has been frequently pointed out in practice and a recent study for NSGA-II gives rigorous runtime results for this behavior~\cite{Zheng2023Inefficiency}.
Consequently, it is essential to provide a theoretical foundation of NSGA-III and MOEA/D which explains their success for many objectives in practice through rigorous analyses. For MOEA/D variants such type of results have been provided in \cite{Huang2021,do2023rigorous} and \neweditx{recently some results have also been provided for the hypervolume-based  SMS-EMOA~\cite{DBLP:conf/aaai/0001D24,WiethegerD24}}. However, to our knowledge up to now no such results are available for \nsgaIII.
We contribute to this line of research by proving analyses of \nsgaIII on classical benchmark problem with $m$ objectives, namely \mLOTZFULL, \mOMMFULL and \mCOCZFULL, which can be shortly written as \mLOTZ, \mOMM and \mCOCZ respectively. 
In $\mLOTZ(x)$ and $\mOMM(x)$ we divide $x$ into \neweditx{$m/2$} many blocks of equal size. In case of $\mLOTZ$ we count the number $\LO$ of leading ones (the length of the longest prefix containing only ones) and the number $\TZ$ of trailing zeros in each block (the length of the longest suffix containing only zeros) while in case of $\mOMM$ we count the number of ones/zeros in each block. In $\mCOCZ$ we divide the second half in $m/2$ blocks of equal size and optimize the number of ones and zeros.

\paragraph{\textbf{Previous related work:}}

The only previous runtime analysis of \nsgaIII was recently published by~\citet{WiethegerD23}. They consider a 3-dimensional variant of \OMM, in which the bit string is cut in two halves and the objectives are to maximize the number of zeros in the whole bit string, to maximize the number of ones in the first half and to maximize the number of ones in the second half:
\[
    \textrm{3-OMM}(x) = \left(n - \sum_{i=1}^n x_i, \sum_{i=1}^{n/2} x_i, \sum_{i=n/2+1}^n x_i\right)
\]
They showed that, when employing $p \ge 21n$ divisions along each objective to define reference points, all individuals associated with the same reference point have the same objective value. This implies that non-dominated solutions can never disappear from the population \neweditx{for a sufficiently large population size}. \Citet{WiethegerD23} proved that \nsgaIII with a population size of $\mu \ge (n/2+1)^2$ with high probability covers the whole Pareto front in $O(cn^3 \log n)$ evaluations. Here $1/c$ is a lower bound on the probability for any individual in the population to produce an offspring in one generation. The smallest value of~$c$ is obtained when using uniform parent selection as then $c = 1-(1-1/\mu)^\mu \ge 1-1/e$ can be chosen and the upper bound becomes $O(\mu n \log n)$ or $O(n^3 \log n)$ for the best choice of the population size~$\mu$.
Our results are not directly comparable as we consider a different definition for \mOMM and we only study even numbers of objectives, $m$. However, we show that in the same time of $O(n^3 \log n)$ \nsgaIII can also solve \mOMM with $m=4$ objectives. 

\paragraph{\textbf{Our contribution:}}
In this paper, we conduct a runtime analysis of NSGA-III, one of the most popular multi-objective optimization algorithms, on well-known benchmark functions \mLOTZ, \mOMM, and \mCOCZ of many-objective optimization. We develop a set of mathematical tools for our analysis which allow general statements about the good behaviors of the algorithm on any many-objective problem, such as being able to maintain its structured population with respect to the reference points. These tools provide parameter settings on how to achieve such behaviors. They are built upon different geometric arguments and calculations compared to the previous work \cite{WiethegerD23}, \neweditx{and are stated more generally.}

\newedit{By applying these tools on the benchmark functions for any constant $m$ number of objectives, we prove the following results: 
NSGA-III with uniform parent selection and standard bit mutation optimizes \mLOTZ in expected $O(n^2)$ generations, or $O(n^{m+1})$ fitness evaluations, using a population size of $\mu=O(n^{m-1})$.
It optimizes \mOMM and \mCOCZ in expected $O(n\log{n})$ generations, or $O(n^{m/2+1}\log{n})$ fitness evaluations, using  population size $\mu=O(n^{m/2})$. The required numbers of reference points are ${4n\sqrt{m}\choose m-1}$
for \mLOTZ and \mOMM, and ${n(m+2)\sqrt{m} \choose m-1}$ for \mCOCZ.}

\newedit{As far as we know this is the first runtime analysis of NSGA\nobreakdash-III on more than three objectives. Our tools can be also applied to the three-objective case
and can greatly reduce the number of required reference points as compared to \cite{WiethegerD23}. However, this is of independent interest and not the main focus of our paper, which is to understand how the running time of NSGA-III scales with \neweditx{the problem dimension for a constant number of objectives}. We believe our work will serve as \neweditx{a stepping stone} towards \neweditx{further analyses} of NSGA-III on more complex problems and towards a better understanding of the capabilities of this advanced EMO algorithm.}



\section{Preliminaries}\label{sec:prelim}

Denote by $\ln$ the logarithm to base $e$ and for $m \in \mathbb{N}$ let $[m]:=\{1, \ldots , m\}$. 
For \neweditx{vectors} $v,\neweditx{w} \in \mathbb{R}^m$ denote by $\lvert{v}\rvert$\neweditx{, $\lvert{w}\rvert$ their} length\neweditx{s (Euclidean norms)},
and 
by $v \circ w$ the\neweditx{ir} standard scalar product, 
i.e. $v \circ w = \lvert v \rvert \lvert w\rvert \cos(\phi)$ where $\phi$ is the angle between $v$ and $w$. 
\neweditx{A \emph{hyperplane} $H$ in $\mathbb{R}^m$ is 
defined by a non-null vector $a\in\mathbb{R}^m$ and a $c\in\mathbb{R}$ as $H:=\{x\in\mathbb{R}^m \mid a\circ x = c\}$.
The \emph{intercept} $I_j$ of $H$ with 
the $j$-th coordinate axis is the set of points $s$ where $s_i=0$ for every $i \neq j$ and $s_j$ satisfying $s_j a_j = c$. Particularly $I_j$ is a singleton when $a_j\neq 0$.} 
For two random variables $X$ and $Y$ on $\mathbb{N}_0$ we say that \neweditx{$Y$ stochastically dominates $X$ if $P(Y \leq c) \leq P(X \leq c)$ for every $c \geq 0$}. 

This paper focuses on multi-objective optimization in
a discrete setting, specifically the maximization of an $m$-objective function $f(x):=(f_1(x), \ldots , f_m(x))$ where $f_i:\{0,1\}^n \to \mathbb{N}_0$ for each $i \in [m]$.

\begin{definition}
Consider an $m$-objective function $f$.
\begin{itemize}
\item Given two search points $x, y \in \{0, 1\}^n$,
$x$ \emph{weakly dominates} $y$, denoted by $x \succeq y$
if $f_{i}(x) \geq f_{i}(y)$ for all $1\leq i \leq m$;
and $x$ \emph{dominates} $y$, denoted $x \succ y$,
if one inequality is strict; if neither $x \succeq y$ nor $y \succeq x$ then $x$ and $y$ are \emph{incomparable}.
\item A set $S \subseteq \{0,1\}^n$ is a \emph{set of mutually incomparable solutions} with respect to $f$ if all search points in $S$ are incomparable (\neweditx{thus any} two search points in $S$ have distinct fitness vectors).
\item Each solution that is not dominated by any other in $\{0, 1\}^n$ is called \emph{Pareto-optimal}. A set of these solutions
that \neweditx{covers} all possible non-dominated fitness values \neweditx{and are mutually incomparable}
is called a \emph{Pareto\neweditx{(}-optimal\neweditx{)} set} of $f$.
\end{itemize}
\end{definition}

The \neweditx{weak} dominance relation and the dominance relation are both transitive and when $m=2$, the function is called \emph{bi-objective}.

Let $f_{\max}$ be the maximum possible value of $f$ in one objective, i.e. $f_{\max}:=\max\{f_j(x) \mid x \in \{0,1\}^n, j \in [m]\}$. 
\newedit{We tacitly assume $f_{\max} \ge 1$ as otherwise $f_j(x) = 0$ for all~$j$ and all~$x$ and optimization is trivial.} 
For $N \subseteq \{0,1\}^n$ let $f(N):=\{f(x) \mid x \in N\}$. 

\nsgaIII~\cite{DebJain2014} is summarized in Algorithm~\ref{alg:nsga-iii}. \neweditx{As common in runtime analysis, we do not specify a stopping criterion as we are only interested in the time until a Pareto-optimal set is found.}
%
%
%
\begin{algorithm2e}[h]
	Initialize $P_0 \sim \Unif( (\{0,1\}^n)^{\mu})$\\
        Initialize \neweditx{$E_0 := \{\{-\infty\}^m\}$, $y^{\max}:= \{-\infty\}^m$, $y^{\min}:=\{+\infty\}^m$ (these are subsequently used in Algorithm~\ref{alg:normalization})}\\
	\For{$t:= 0$ to $\infty$}{
		Initialize $Q_t:=\emptyset$\\
		\For{$i=1$ to $\mu$}{
			Sample $p$ from $P_t$ uniformly at random \label{line:nsga-iii:selection}\\
			Create $x$ by standard bit mutation on $p$ with mutation probability $1/n$\\
			Update $Q_t:=Q_t \cup \{x\}$\\
		}
		Set $R_t := P_t \cup Q_t$\\
		Partition $R_t$ into layers $F^1_t,F^2_t,\ldots ,F^k_t$ of non-dominated fitness values\\
            Find $i^* \geq 1$ such that $\sum_{i=1}^{i^*-1} \lvert{F_t^i}\rvert < \mu$ and $\sum_{i=1}^{i^*} \lvert{F_t^i}\rvert \geq \mu$\\
            Compute $Y_t = \bigcup_{i=1}^{i^*-1} F_t^i$\\
            Select $\tilde{F}_t^{i^*} \subset F_t^{i^*}$ such that $\lvert{Y_t \cup \tilde{F}_t^{i^*}}\rvert = \mu$ via \textsc{Select} \neweditx{(Algorithm~\ref{alg:Survival-Selection}, which also calls Algorithm~\ref{alg:normalization})}\\  
		Create the next population $P_{t+1} := Y_t \cup \tilde{F}^t_{i^*}$\\
	}
	\caption{NSGA-III Algorithm on an $m$-objective function $f$ with population size $\mu$ \neweditx{~\cite{DebJain2014,WiethegerD23}}}
	\label{alg:nsga-iii}
\end{algorithm2e}
At first, a population of size $\mu$ is generated by initializing $\mu$ individuals uniformly at random. Then in each generation, a population $Q_t$ on $\mu$ new offspring is created \neweditx{by $\mu$ times choosing a parent $p$ uniformly at random} in Line~\ref{line:nsga-iii:selection}, and then applying standard bit mutation on $p$ where each bit is flipped independently with probability $1/n$. During the survival selection,
the parent and offspring populations $P_t$ and $Q_t$ are joined into $R_t$,
and then partitioned into
layers $F^1_{t+1},F^2_{t+1},\dots$ by the \emph{non-dominated sorting algorithm}~\cite{Deb2002}.
The layer $F^1_{t+1}$ consists of all non-dominated points,
and $F^i_{t+1}$ for $i>1$ contains points that are only dominated by
those from $F^1_{t+1},\dots,F^{i-1}_{t+1}$. Then the critical and unique index $i^*$ with $\sum_{i=1}^{i^*-1} \lvert{F_i^t}\rvert < \mu$ and $\sum_{i=1}^{i^*} \lvert{F_i^t}\rvert \geq \mu$ is determined (i.e. there are fewer than $\mu$ search points in $R_t$ with a lower rank than \neweditx{$i^*$}, but at least $\mu$ search points with rank at most $i^*$). 
Then all individuals with a smaller rank than $i^*$ are taken into $P_{t+1}$ and the remaining points are chosen from $F_{i^*}^t$ with the procedure \textsc{Select} (see Algorithm~\ref{alg:Survival-Selection}): Firstly, a normalized objective function $f^n$ is computed (see Algorithm~\ref{alg:normalization}) and secondly the search points from $F_t^{i^*}$ are associated \neweditx{with} reference points, which are distributed in a uniform way on the normalized hyperplane, i.e. the plane given by the equation $x_1 + \ldots + x_m = 1$. 

We use the detailed normalization from~\cite{WiethegerD23} with minor differences to ~\cite{Blank2019}. \neweditx{Although this procedure is formulated for minimization problems, we show that it is also suitable for maximization problems if the given positive threshold $\varepsilon_{\text{nad}}$ is sufficiently large.}
Let $t$ be any generation. Let $y_j^{\min}$ and $y_j^{\max}$ be the minimum and maximum value in the $j$-th objective over all joined populations from generation $t$ and previous ones (i.e. from $R_0, \ldots R_t$),   respectively. 
%
The point $y^{\min}:=(y^{\min}_1, \dots, y^{\min}_m)$ is called \neweditx{\emph{the ideal point} in \cite{DebJain2014}}. 
Then for each objective $j$ we compute an extreme point \neweditx{$e^{(j)}$} in that objective by using an achievement scalarization function (see \neweditx{Section 4.2}  of~\cite{Blank2019} for the exact definition), \neweditx{and the hyperplane spanned by these $e^{(j)}$ is denoted $H$}.
The intercept $I_j$ of $H$ 
with the $j$-th objective axis gives the \neweditx{so-called} \emph{Nadir point estimate }$y_j^{\text{nad}}$. 
\neweditx{If} $H$ has smaller dimension than $m$ or if an intercept is either larger than a given positive threshold $\varepsilon_{\text{nad}}$ or smaller than $y_j^{\max}$, $y_j^{\text{nad}}$ is set to $\max_{y\in F_t^1} f_j(y)$, the maximum of the $j$-th objective values which occur in the current first layer $F_t^1$ (which may differ from $y_j^{\max}$). We also replace $y_j^{\text{nad}}$ by $\max_{y\in F_t^1} f_j(y)$ if the Nadir point estimate is smaller than $y_j^{\min} + \varepsilon_{\text{nad}}$. Then the normalization $f_j(x)^n$ is computed as in Line~\ref{line:normalization-value} in Algorithm~\ref{alg:normalization} and the normalized objective function $f^n$ is defined as $(f_1^n , \ldots , f_m^n)$. 
\newedit{The goal of normalization is to have the normalized fitness vectors in \neweditx{$(\R_0^+)^m$} or in $[0,1]^{m}$ while preserving or properly emphasizing the dominance relations \cite{Blank2019}.}
Note that an individual $x$ (weakly) dominates another individual $y$ with respect to $f$ if and only if $x$ (weakly) dominates $y$ with respect to $f^n$.
\begin{figure}
    \centering
    \includegraphics[scale=0.8]{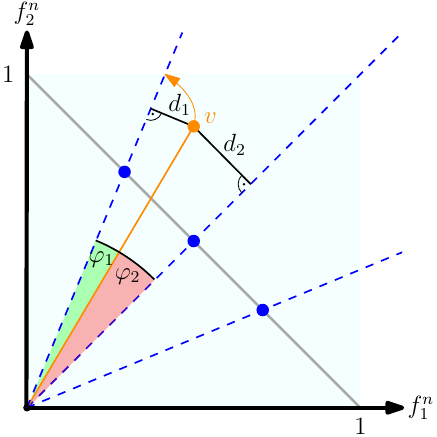}
    \caption{Sketch of the way search points (orange) are associated with reference points (blue dots connected by dashed blue lines from the origin) for $m=2$ dimensions. The axes show the normalized fitness and all points lie in the unit cube $[0, 1]^m$ (blue shading). In the example, $\varphi_1 < \varphi_2$ and via $\sin(\varphi_1) = d_1/\vert{v}\vert$ and $\sin(\varphi_2) = d_2/\vert{v}\vert$ we conclude $d_1 < d_2$; the search point~$v$ is associated with the leftmost reference point.}
    \label{fig:reference-points} 
\end{figure}
\neweditx{
\begin{example}
    Suppose that $m=3$ and $e^{(1)}=(2,1,2)$, $e^{(2)} = (1,0,3)$ and $e^{(3)} = (0,2,0)$. 
    Then $H$ is given by the equation $-x_1+4x_2+3x_3=8$ and the intercept of $H$ with the first coordinate axis is $-8$ which is negative. Hence valid is set to \textit{false} and we have $y_j^{\text{nad}} = \max_{y \in F_1}f_j(y)$. On the other hand if $y^{\max}=(2,4,2)$, $e^{(1)}=(2,0,0)$, $e^{(2)} = (0,4,0)$ and $e^{(3)} = (0,0,2)$ then $H$ is given by $2x_1+x_2+2x_3=4$ and then the intercepts are $I_1=2$, $I_2=4$ and $I_3=2$. Then valid is set to \textit{true} and we have then $y_j^{\text{nad}} = I_j$ which coincides also with $y_j^{\max}$. Note that $y_j^{\max} \neq \max_{y \in F_1}f_j(y)$ in general.
\end{example}}
\newedit{In Algorithm~\ref{alg:Survival-Selection} we use the same set of structured reference points as
proposed in the original paper~\cite{DebJain2014}, which stems from the method of
Das~and~Denis~\cite{Das1998}. The points are defined on the simplex of the unit
vectors 
$(1,0,\dots,0)^{\intercal},(0,1,\dots,0)^{\intercal},\dots,(0,0,\dots,1)^{\intercal}$ as:}
\[
\refer:=\left\{\left(\frac{a_1}{p}, \ldots ,\frac{a_m}{p} \right) 
    \mid  
    \newedit{(a_1,\dots,a_m) \in \mathbb{N}_0^m,} 
    \sum_{i=1}^m a_i = p
\right\}
\]
where $p \in \mathbb{N}$ is a parameter one can choose according 
to the fitness function $f$. The cardinality of $\refer$,
\ie the number of ways to choose non-negative integers $(a_1,\dots,a_m)$ summing up to $p$, is ${p+m-1 \choose m-1} = {p+m-1 \choose p}$ which is 
the \emph{weak} $m$-\emph{composition number} of $p$\neweditx{. This can be shown using the stars and bars argument, \cf Chapter II.5 of \cite{Feller1968}.}

After \neweditx{the} normalization procedure, each individual of rank at most $i^*$ is associated \neweditx{with} the reference point $\text{rp}(x)$ such that the angle between $f^n(x)$ and $\text{rp}(x)$ is minimized (see Algorithm~\ref{alg:Survival-Selection} and Figure~\ref{fig:reference-points}). Equivalently, $x$ is associated \neweditx{with} that reference point such that the distance between $f^n(x)$ and the line through the origin and $\text{rp}(x)$ is minimal. \neweditx{Here ties are broken deterministically, i.e. two individuals which have the same smallest distance to two reference points are assigned to the same reference point.} Then, one iterates through all the reference points where \newedit{the} reference point with the fewest associated individuals that are already selected for the next generation $P_{t+1}$ is chosen. Ties are broken uniformly at random. A reference point is omitted if it only has associated individuals that are already selected for $P_{t+1}$. Then, among the not yet selected individuals \neweditx{of that reference point} the one nearest to the chosen reference point is taken for the next generation where ties are again broken \neweditx{uniformly} at random. If the required number of individuals is reached (i.e. if $\lvert{Y_t}\rvert+\lvert{\tilde{F}_t^i}\rvert = \mu$) the selection ends. 
Two minor differences to the treatment of Blank et al. in~\cite{Blank2019} are that our $H$ is spanned by 
$e^{(j)}$ and not by $e^{(j)}-y^{\min}$ 
and that \neweditx{we set} $y_j^{\text{nad}} := I_j$ 
in Line~\ref{line:intercept} 
\neweditx{instead of} $y_j^{\text{nad}} := I_j+y_j^{\min}$ (see also~\cite{WiethegerD23}).

\begin{algorithm2e}[h]
        Compute the normalization $f^n$ of $f$ using the procedure \textsc{Normalize} (see Algorithm~\ref{alg:normalization})\\
        Associate each $x\in Y_t \cup F_t^{i^*}$ to its reference point $\rp(x)$ based on the smallest distance to the reference rays\\ 
        For each $r \in \refer$, initialize $\rho_r:=|\{x\in Y_t \mid \mathrm{rp}(x)=r\}|$\\
        Initialize $\tilde{F}_t^{i^*}=\emptyset$ and $R':=\refer$\\
        \While{true}{
        Determine $r_{\min} \in R'$ such that $\rho_{r_{\min}}$ is minimal (where ties are broken randomly)\\
        Determine $x_{r_{\min}} \in F_t^{i^*} \setminus \tilde{F}_t^{i^*}$ which is associated with $r_{\min}$ and minimizes \newedit{the distance} between the vectors $f^n(x_{r_{\min}})$ and $r_{\min}$ (where ties are broken randomly)\\
        \If{$x_{r_{\min}}$ exists}{
        $\tilde{F}_t^{i^*} = \tilde{F}_t^{i^*} \cup \{x_{r_{\min}}\}$\\
        $\rho_{r_{\min}} = \rho_{r_{\min}} + 1$\\
            \lIf{$\lvert{Y_t}\rvert + \lvert{\tilde{F}_t^{i^*}}\rvert = \mu$}
                {\Return{$\tilde{F}_t^{i^*}$}
                }
            }
            \lElse{$R'=R' \setminus \{r_{\min}\}$}
        }
	\caption{Selection procedure \textsc{Select} based on a set $\refer$ of reference points for maximizing a function}
	\label{alg:Survival-Selection}
\end{algorithm2e}

\begin{algorithm2e}[h]
        $E_{t+1}=\emptyset$\\
	\For{$j = 1$ to $m$}{
		$y^{\min}_j = \min\{y^{\min}_j, \min_{y \in R_t} (f_j(y))\}$\\ 
            $y^{\max}_j = \max\{y^{\max}_j, \max_{y \in F_t^1} (f_j(y))$\}\label{line:max}\\ 
            Determine an extreme point $e^{(j)}$ in the $j$-th objective from $\neweditx{f(}Y_t \cup F_t^{i^*}\neweditx{)} \cup E_t$ using an achievement scalarization function~\label{line:Extreme-points}\\
            $E_{t+1} = E_{t+1} \cup \{e^{(j)}\}$}
            valid = \textit{false}\\
            \If{$e^{(1)}, \ldots , e^{(m)}$ are linearly independent~\label{line:lin-independent}}{
             valid = \textit{true}\\
             Let $H$ be the hyperplane spanned by $e^{(1)}, \ldots , e^{(m)}$\\
            \For{$j=1$ to $m$}{
            Determine the intercept $I_j$ of $H$ with the $j$-th objective axis\\
            \lIf{$I_j \geq \varepsilon_{\text{nad}}$ and $I_j \leq y^{\max}_j$~\label{line:Second-If}}{
            $y_j^{\text{nad}} = I_j$\label{line:intercept}}
            \Else{
            valid = \textit{false}\label{line:False}\\
            \textbf{break}
                    }
                }
            }
		\If{valid = \textit{false}}{
            \For{$j=1$ to $m$}{
            $y_j^{\text{nad}} = \max_{y \in F_1}f_j(y)$~\label{line:valid-false}
                }
            }
            \For{$j=1$ to $m$}{
            \If{$y_j^{\text{nad}}<y^{\min}_j + \varepsilon_{\text{nad}}$}{
            $y_j^{\text{nad}} = \max_{y \in F_1 \cup \neweditx{\ldots} \cup F_k}f_j(y)$~\label{line:normalization-value}
                }
            }
         Let $f_j^n(x) = (f_j(x)-y^{\min}_j)/(y_j^{\text{nad}}-y^{\min}_j)$ for each $x \in \{0,1\}^n$ and $j \in \{1, \ldots , m\}$ \label{line:output}
	\caption{Normalization procedure \textsc{Normalize} \cite{DebJain2014,WiethegerD23}}
	\label{alg:normalization}
\end{algorithm2e}

%

\section{Sufficiently Many Reference Points Protect Good Solutions}


The major goal is to show that, by employing sufficiently many reference points, once the population covers a fitness vector with a solution $x \in F_t^1$, it is covered for all future generations as long as $x \in F_t^1$. Such a statement was already shown in the groundbreaking paper by \citet{WiethegerD23}, but this was limited to the 3-objective \neweditx{\mOMM} problem 
\neweditx{and some} precision \neweditx{was} 
los\neweditx{t}
when \neweditx{they applied 
the} 
inequality 
$\arccos(1-1/24)>2\arccos(1-18/42^2)$. We give a general proof that \nsgaIII maintains good solutions using different analytical arguments and require only $p=2 \cdot 3^{3/2}n \thickapprox 10\neweditx{.}39n$ divisions in case of 3-\OMM, which is approximately by a factor of $2$ smaller. \newedit{In contrast to~\cite{WiethegerD23} we compute distances from search points to their nearest reference point via the sine function of the angle between the lines running through both point\neweditx{s and} use a general trigonometrical identity for $\arcsin(x)$ for any $0<x\leq 1$.}

\newedit{We start with a lemma about the normalized objectives of a multiobjective function.}

\begin{lemma}
\label{lem:Normalization}
Consider \nsgaIII on a multiobjective function $f:\{0,1\}^n \to \mathbb{N}_0^m$ \neweditx{and suppose that $\varepsilon_{\mathrm{nad}} \geq f_{\max}$}. \neweditx{In every generation, after running Algorithm~\ref{alg:normalization},} for every $x \in R_t$ and every objective $j$ we have $0 \leq f_j^n(x) \leq 1$ and $y_j^{\mathrm{nad}}-y_j^{\min} \leq f_\max$. 
\end{lemma}

\begin{proof}
\neweditx{Until Line~\ref{line:normalization-value} of the algorithm, we note that $y_j^{\mathrm{nad}}$ is either set to (i) $\max_{y \in F_t^1} f_j(y)$ or (ii) the intercept $I_j$. 
    For (i), since $f_{\max} \geq \max_{y \in F_t^1} f_j(y)$, $y_j^{\min} \geq 0$ and $f_j(x) \leq \max_{y \in F_t^1}f_j(y)$
    we get 
    \begin{align}
    (y_j^{\mathrm{nad}}-y_j^{\min} \leq f_\max)  
    \wedge (f_j(x) \leq y_j^{\mathrm{nad}})
    \label{eq:proper-normalization}
    \end{align}
    and the results follow because the next for loop can only increase $y_j^{\max}$ as more layers are considered. 
    In the case of (ii), the setting is triggered by Line~\ref{line:Second-If} of the algorithm, thus the condition of the if statement must be true. Specifically 
    $I_j \geq \varepsilon_{\mathrm{nad}} \geq f_{\max}$ and $I_j \leq y_j^{\max} \leq f_{\max}$ and this implies $y_j^{\mathrm{nad}}=I_j = f_{\max}$. If the condition in the next for loop is also triggered then we fall back to case (i), otherwise note that \eqref{eq:proper-normalization} still holds because $f_j(x)\leq f_{\max}$.}    
\end{proof}

\neweditx{The aim for the rest of this section is to generalize Lemma~2 in~\cite{WiethegerD23} to arbitrary objective functions with an arbitrary number of objectives. We start with the following technical result about general \neweditx{vectors} $v,w \in [0,1]^m$.}
\neweditx{\begin{lemma}
\label{lem:angle-helper}
    Let $v,w \in [0,1]^m$, and suppose that there are indices $i, j$ with $w_i=v_i+\ell_1$ and $w_j=v_j-\ell_2$ where $\ell_1,\ell_2 \geq 1/f_{\max}$. 
    Denote by $\varphi$ the angle between $v$ and $w$. Then $\sin(\varphi) \geq 1/(\sqrt{m} f_{\max})$.
\end{lemma}}
\begin{proof}
    \neweditx{Without loss of generality we may assume that $i=1$ and $j=2$. Since $v$ and $w$ have non-negative components we have that $v \circ w \geq 0$ and hence $\cos(\varphi) \in [0,1]$ which implies $\varphi \in [0, \pi/2]$. Thus, $\sin(\varphi) = q/\lvert{v}\rvert$ where $q$ is the distance of the point $v$ to the line $g$ through the origin and the point $w$. 
    Note that $g$ is defined as $g=\beta \cdot w$ for $\beta \in \mathbb{R}$. We estimate $q$: let $h$ be the line through $v$, intersecting $g$, and which is orthogonal to $g$, i.e. $h=v+\lambda \cdot t$ for $\lambda \in \mathbb{R}$ and a suitable $t \in \mathbb{R}^{m}$. The intersection point of $g$ and $h$ is $\mu \cdot w$ for a suitable $\mu \in \mathbb{R}$ and thus the distance $q$ of $v$ to this intersection point is $q=\lvert{v-\mu w}\rvert$. We obtain
    \begin{align*}
    q&=\lvert{(v_1, \ldots , v_m) - \mu \cdot (v_1 + \ell_1, v_2 - \ell_2,w_3, \ldots ,w_m)}\rvert \\
    &=\lvert{((1-\mu)v_1 - \mu\ell_1, (1-\mu)v_2 + \mu\ell_2,v_3 - \mu w_3, \ldots ,v_m - \mu w_m)}\rvert.
    \end{align*}
    If $\mu \geq 1$ then $q \geq 1/f_{\max}$ as $(1-\mu)v_1 - \mu\ell_1 \leq - \ell_1 \leq -1/(f_{\max})$ and if $\mu \leq 1$ then $q \geq (1-\mu)v_2 + \mu \ell_2 \geq \ell_2 \geq 1/f_{\max}$ as $v_2-\ell_2 \geq 0$ implies $v_2 \geq \ell_2$. This implies $\sin(\varphi) \geq 1/(\lvert{v}\rvert f_{\max}) \geq 1/(\sqrt{m} f_{\max})$ since $\lvert{v}\rvert \leq \sqrt{m}$ due to $\lvert{v_i}\rvert \leq 1$ for all $i \in [m]$.}
    \end{proof}

\begin{lemma}
\label{lem:Same-Reference}
    Consider \nsgaIII on an $m$-objective function~$f$ \neweditx{with $\varepsilon_{\text{nad}} \geq f_{\max}$}. 
    If \newedit{the set of reference points $\refer$ is created using $p$ divisions along each objective and} $p \geq 2m^{3/2} f_{\max}$, 
    all search points in $F_t^1$ that are associated with the same reference point from $\refer$ have the same fitness vector.
\end{lemma}

\begin{proof}
An individual $x$ is associated \neweditx{with} the reference point $r$ that minimizes the distance between $f^n(x)$ and the line through the origin and $r$, (i.e. that minimizes the angle between this point and $f^n(x)$, cf.\ Figure~\ref{fig:reference-points}). To show the statement, we first 
    \begin{itemize}
        \item[(i)] lower bound the angle between two individuals from $F_t^1$ with different fitness vectors by 
        $\arcsin(1/(\sqrt{m} f_{\max}))$ and 
        \item[(ii)] upper bound the angle between a normalized fitness vector of $f$ and its nearest reference point by $\arcsin(1/(2\sqrt{m}f_{\max}))$.
    \end{itemize}    
    Hence, the lower bound in (i) is more than \neweditx{twice 
    the upper bound} in (ii) 
    \neweditx{as} $\arcsin(1/(\sqrt{m} f_{\max})\neweditx{)}/\arcsin(1/(2\sqrt{m}f_{\max})) > 2$ by Lemma~\ref{lem:Trigonometry} for $x=1/(\sqrt{m}f_{\max})$.  
Then search points with different fitness vectors are never associated with the same reference point, because if they were then the angle between their normalized fitness vectors would be at most the sum of their angles to that reference point.

    (i): Let $x,z \in F_t^1$ be two search points with distinct fitness. 
\newedit{As $x$ and $z$ are incomparable and the domain of $f$ is $\mathbb{N}_0^m$,} the fitness vectors $f(x)$ and $f(z)$ differ by at least $1$ in two objectives before normalization. After normalization, the difference of $f^n(x)$ and $f^n(z)$ correspond to at least $1/f_{\max}$ 
in those objectives since $y_i^{\text{nad}}-y_i^{\text{min}} \leq f_{\max}$ by Lemma~\ref{lem:Normalization} \neweditx{for $i \in [m]$}. Thus there are $i,j \in [m]$ such that $f^n(z)_i=f^n(x)_i+\ell_1 \geq 0$ and $f^n(z)_j = f^n(x)_j - \ell_2 \geq 0$ for $\ell_1,\ell_2 \geq 1/f_{\max}$. By Lemma~\ref{lem:angle-helper} we have that $\sin(\varphi) \geq 1/(\sqrt{m}f_{\max})$ where $\varphi$ is the angle between $f^n(x)$ and $f^n(z)$ (since  \neweditx{$f^n(x),f^n(y) \in [0,1]^m$} by  Lemma~\ref{lem:Normalization}). Since $\arcsin$ is increasing and $\varphi \in [0,\pi/2]$ (owing to $f^n(x) \circ f^n(z) \geq 0$ we have that $\cos(\varphi) \in [0,1]$), we obtain $\varphi \geq \arcsin(1/(\sqrt{m}f_{\max}))$.
    
(ii): We scale $f^n(x)$ to $t = a \cdot f^n(x)$ for $a>0$ so that $t_1 + \ldots + t_m = 1$. We claim that there is $r \in \refer$ with $\lvert{t_i-r_i}\rvert \leq 1/p$ for every $i \in [m]$ which implies $\lvert{t-r}\rvert \leq \sqrt{m}/p$.
    At first choose $b_1, \ldots , b_m \in \mathbb{N}_0$ with $t_i \in [b_i/p,(b_i+1)/p]$. Since $t_1 + \ldots + t_m = 1$, we see $\sum_{i=1}^m b_i/p \leq 1$ and $\sum_{i=1}^m (b_i+1)/p \geq 1$. Hence, $\sum_{i=1}^m b_i \leq p$ and $\sum_{i=1}^m (b_i+1) \geq p$. Thus, there is $\ell \in [m]$ with $\left(\sum_{i=1}^\ell b_i\right) + \sum_{i=\ell+1}^m (b_i+1) = p$. So \neweditx{for $i \in [m]$} choose $r_i=b_i/p$ if $i\leq \ell$ and $r_i=(b_i+1)/p$ if $i > \ell$. Then $\lvert{t_i-r_i}\rvert \leq 1/p$ and $(r_1, \ldots , r_m) \in \refer$ since $r_1 + \ldots + r_m = 1$.
    Now let $\vartheta$ be the angle between $t$ and $r$. Then $\sin(\vartheta) = q/\lvert{t}\rvert$ where $q$ is the distance between the point $t$ and the line through the origin and $r$. Note that $q \leq \lvert{t-r}\rvert \leq \sqrt{m}/p$ and $\lvert{t}\rvert \geq 1/\sqrt{m}$ (since $t_1 + \ldots + t_m=1$). Thus $\sin(\vartheta) \leq \sqrt{m} \cdot \sqrt{m}/p = m/p \leq 1/(2\sqrt{m}f_{\max})$. 
    Since $\arcsin$ is increasing, this yields that the angle between $t$ and $r$ (and hence also the angle between $f^n(x)$ and $r$) is at most $\arcsin(m/p) \leq \arcsin(1/(2\sqrt{m}f_{\max}))$ (because $r \circ t \geq 0$, we have that $\cos(\vartheta) \in [0,1]$, i.e. $\vartheta \in [0,\pi/2]$).
\end{proof}


\newedit{Arguing as in~\cite{WiethegerD23}, the fact that search points assigned to different reference points have different fitness vectors implies that for every first ranked individual $x$ there is an individual $y$ with the same fitness vector which is taken into the next generation.}
\neweditx{The number of reference points is $|\refer|=\binom{p+m-1}{m-1}$. So for example, if $m=O(1)$ and we set $p = 2m^{3/2} f_{\max} = \Theta(f_{\max})$, then $|\refer|=\Theta(f_{\max}^{m-1})$.}

\begin{lemma}
\label{lem:Reference-Points}
    Consider \nsgaIII optimizing an $m$-objective function $f$ with \neweditx{$\varepsilon_{\text{nad}} \geq f_{\max}$ and} a set $\refer$ of reference points as defined above for $p \in \mathbb{N}$ with $p \geq 2m^{3/2}f_{\max}$. Let $P_t$ be its current population and $F_t^1$ be the multiset describing the first layer of the joint population of parent and offspring. Assume the population size $\mu$ \neweditx{fulfills the condition} \neweditx{$\mu \geq \lvert{S}\rvert$} where $S$ is a maximum set of mutually incomparable solutions. 
    Then for every $x \in F_t^1$ there is a $x' \in P_{t+1}$ with $f(x')=f(x)$.
\end{lemma}
\begin{proof}
\neweditx{Note that $\lvert{S}\rvert \geq \lvert\{f(x) \mid x \in F_t^1\}\rvert$. Hence, it suffices to prove the lemma for $\mu \geq \lvert\{f(x) \mid x \in F_t^1\}\rvert$.}
Suppose that $x \in F_t^1$. 
We show that a solution $x'$ with $f(x')=f(x)$ is taken into $P_{t+1}$. If $\lvert{F_t^1}\rvert \leq \mu$
all individuals in $F_t^1$ including $x$ survive. Otherwise, 
the objective functions with respect to $F_t^1$ are normalized. 
Let $\ell \coloneqq \lvert{\{f(x) \mid x \in F_t^1\}\rvert}$
be the number of different fitness vectors of individuals from $F_t^1$. 
\neweditx{Given the condition on $p$, Lemma~\ref{lem:Same-Reference} implies that individuals with distinct fitness vectors in $F^1_t$ are associated \neweditx{with} different reference points. 
Thus} there are 
\neweditx{as many as} $\ell$ reference points to which at least one individual is associated. 
\neweditx{Since $\mu\geq \ell$, at least} 
one individual $x' \in R_t$ associated \neweditx{with} the same reference point as $x$ survives and we have $f(x)=f(\tilde{x})$.
\end{proof}



\section{$m$-LOTZ}
In this section, we derive the first result on the runtime of \nsgaIII on the $\mLOTZ$ benchmark. In $\mLOTZ(x)$\neweditx{,} we divide $x$ in $m/2$ many blocks and in each block we count the $\LO$-value (the length of the longest prefix containing only ones in this block) and the $\TZ$-value (the length of the longest suffix containing only zeros in this block). This can be formalized as follows.
\begin{definition}[\citet{Laumanns2004}]
\label{def:mLOTZ}
Let $m$ be divisible by $2$ and let the problem size be a multiple of $m/2$. Then the $m$-objective function \mLOTZ is defined by
$\mLOTZ: \{0,1\}^n \to \mathbb{N}_0^m$ as 
\[
\mLOTZ(x) = (f_1(x), f_2(x), \ldots ,f_m(x))
\]
with 
\[
f_k(x)=
\begin{cases}
    \sum_{i=1}^{2n/m} \prod_{j=1}^i x_{j+n(k-1)/m}, & \text{ if $k$ is odd,} \\
    \sum_{i=1}^{2n/m} \prod_{j=i}^{2n/m} (1-x_{j+n(k-2)/m}), & \text{ else,}
\end{cases}
\]
for all $x=(x_1, \ldots ,x_n) \in \{0,1\}^n$.
\end{definition}

A Pareto-optimal set of \mLOTZ is
\[
\{1^{i_1}0^{2n/m-i_1} \ldots 1^{i_{\neweditx{m/2}}}0^{2n/m-i_{\neweditx{m/2}}} \mid i_1, \ldots , i_{\neweditx{m/2}} \in \{0, \ldots , 2n/m\}\}
\]
which coincides with the set of Pareto-optimal search points of \mLOTZ. The cardinality of this set is $(2n/m+1)^{m/2}$. 

With Lemma~\ref{lem:Reference-Points} 
we see that mutually incomparable solutions from the first rank are not being lost in future
generations if the number of reference points is sufficiently high and the population size is at least as large as the maximum cardinality set of mutually incomparable solutions $S$. This is similar to the NSGA-II algorithm, for which previous work showed that the population size must be chosen large enough to guarantee the survival of all mutually incomparable solutions~\cite{Dang2023,DaOp2023,Dang2024,DoerrQu22,Doerr2023,ZhengLuiDoerrAAAI22,ZhengD2023}. 

\neweditx{To bound}
the expected runtime on \mLOTZ, we first bound the cardinality of $S$ as follows. \newedit{
\neweditx{This cardinality} 
can be much larger than the size of the Pareto front, which is $(2n/m+1)^{m/2}$.}

\begin{lemma}
\label{lem:fitnessvectors-non-dom-LOTZ}
    Let $S$ be a maximum cardinality set of mutually incomparable solutions for $f:=\mLOTZ$. Then 
    \[
    \lvert{S}\rvert \leq (2n/m+1)^{m-1}.
    \]
\end{lemma}
\begin{proof} 
Let $V:=\{f(x) \mid x \in S\}$. Then for two distinct $u,v \in V$ it holds $(u_1, \ldots , u_{m-1}) \neq (v_1, \ldots ,v_{m-1})$ (otherwise $x,y$ with $f(x)=u$ and $f(y)=v$ would be comparable) and hence $\lvert{S}\rvert = \lvert{V}\rvert \leq {(2n/m+1)^{m-1}}$.
\end{proof}

\newedit{One can also show that the inequality in Lemma~\ref{lem:fitnessvectors-non-dom-LOTZ} is asymptotically tight. Due to space restrictions, we do not give a proof here.}
Now we are able to show that \nsgaIII can find the Pareto front of $\mLOTZ$ efficiently. 

\begin{theorem}
\label{thm:Runtime-Analysis-NSGA-III-mLOTZ}
Consider \nsgaIII with uniform parent selection
optimizing $f:=\mLOTZ$ for any \neweditx{constant} $m \in \mathbb{N}$ divisible by $2$ with \neweditx{$\varepsilon_{\text{nad}} \geq 2n/m$ and} a set $\refer$ of reference points as defined above for $p \in \mathbb{N}$ with $p \geq 4n\sqrt{m}$, a population size $\mu \geq (2n/m+1)^{m-1}$, and $\varepsilon_{nad}>2n/m$. 
For every initial population, a whole Pareto-optimal set of \mLOTZ is found in $6n^2$ generations and $6\mu n^2$ fitness evaluations with probability $1-e^{-\Omega(n)}$. The expected number of evaluations is at most $6\mu n^2 + o(1)$.
\end{theorem}

\begin{proof}
    By Lemma~\ref{lem:fitnessvectors-non-dom-LOTZ} 
    the condition on the population size $\mu$ in Lemma~\ref{lem:Reference-Points} is always met. Along with $f_{\max}=2n/m$ and $p \geq 4n\sqrt{m} = 2m^{3/2}f_{\max}$, Lemma~\ref{lem:Reference-Points} is applicable in every generation $t$.
    
    We use the method of typical runs and divide the optimization procedure into two phases. We  
    show that with probability $(1-e^{-\Omega(n)})$ each phase is completed in $3n^2$ generations.
    
    \textbf{Phase $1$:} Create a Pareto-optimal search point.\\
    We upper bound the probability by $e^{-\Omega(n)}$ that a Pareto-optimal search point is not created in $3n^2$ generations. For $x \in P_t$ define the vector $w(x) \in \{0, \ldots , 2n/m\}^{m/2}$ by $w(x)_k:=f_{2k-1}(x)+f_{2k}(x)$. Note that $w(x)_k = 2n/m$ for every $k \in \{1, \ldots , n\}$ if and only if $x$ is Pareto-optimal. 
    Set $g_t:=\max_{x \in P_t}\sum_{i=1}^{m/2} w(x)_i \in [n] \cup \{0\}$. Then there is a Pareto-optimal solution if and only if $g_t=n$. Since a search point $x$ can only be dominated by a search point $y$ if $\sum_{k=1}^{m/2} w(x)_k<\sum_{i=1}^{m/2} w(y)_k$, $g_t$ cannot decrease by Lemma~\ref{lem:Reference-Points}. 
    For $k \in \{0\} \cup [n-1]$ we define the random variable $X_k$ as the number of generations $t$ with $g_t=k$. Then the number of generations until there is a Pareto-optimal solution is at most $\sum_{k=0}^{n-1} X_k$. To increase $g_t$, it suffices to choose an individual $y$ with $\sum_{k=1}^{m/2} w(y)_k=g_t$ and to create an offspring $z$ with $w(z)_r > w(y)_r$ for one $r \in [m/2]$ while $w(z)_q = w(y)_q$ for all $q \in [m/2] \setminus \{r\}$. The former happens with probability at least $1/\mu$ and for the latter we have to flip a specific bit of $y$ during mutation (in the $r$-th block of $y$) and keep the other bits unchanged, which happens with probability $1/n \cdot (1-1/n)^{n-1} \geq 1/(en)$. Let $s:=1/(en \mu)$. Hence, in one generation (i.e. during $\mu$ offspring productions), the probability of increasing $g_t$ is at least 
    \[
    1-(1-s)^{\mu} \geq \frac{s \mu}{s\mu + 1}:= \sigma
    \]
    due to Lemma~10 in~\cite{Badkobeh2015}. Hence, \neweditx{$X_k$ is stochastically dominated by a geometric random variable $Y_k$ with success probability $\sigma$} (i.e. for any $\lambda \geq 0$ we have $P(Y_k \leq \lambda) \leq P(X_k \leq \lambda)$). Note that $E[Y_k] = 1/\sigma = 1+en$. Define $Y:= \sum_{k=0}^{n-1} Y_k$. Then $E[Y]=n+en^2$ and since the $Y_k$ can be seen as independent 
    we have by a theorem of Witt 
    (see Theorem~1 in~\cite{WITT201438})
    \begin{align*}
    P(X \geq 3n^2) &\leq P(Y \geq 3n^2) = P(Y \geq E[Y]+\lambda) \\
    &\leq \exp(-\min(\lambda^2/r,\lambda \sigma)/4) = \exp(-\Omega(n)),
    \end{align*}
    where $\lambda:=(3-e)n^2-n \in \Theta(n^2)$ and $r:=\sum_{i=0}^{n-1}(1/\sigma)^2 = \sum_{i=1}^n(en+1)^2 \in \Theta(n^3)$. Hence, the probability of $g_t < n$ (i.e. $P_t$ contains no Pareto-optimal solution $x$) after $3n^2$ generations is at most $e^{-\Omega(n)}$.
    To bound the expectation, \neweditx{if} the goal is not reached after $3n^2$ generations, we repeat the above arguments with another period of $3n^2$ generations. The expected number of periods required is $1 + e^{-\Omega(n)}$.
    

    \textbf{Phase $2$:} Cover the whole Pareto front.

    Note that, by Lemma~\ref{lem:Reference-Points}, Pareto-optimal fitness vectors can never disappear from the population.
    We upper bound the probability by $e^{-\Omega(n)}$ that not the whole Pareto front is covered after $3n^2$ generations. For a specific fitness vector $v \in \mathbb{N}_0^m$ with $v_{2i-1}+v_{2i} = 2n/m$ \newedit{for $i \in [m/2]$} we first upper bound the probability that a solution $x$ with $f(x)=(v_1, \ldots ,v_m)$ has not been created after $3n^2$ generations. Let $y \in \{0,1\}^n$ be with $f(y)=v$. We consider the distance of~$y$ to the closest Pareto-optimal search point in the population. For each generation $t$ let $d_t:=\min_{x \in \mathcal{F}_m \cap P_t}H(x,y)$ where $\mathcal{F}_m$ denotes the set of all Pareto-optimal search points (i.e. $\mathcal{F}_m \cap P_t$ is the set of all Pareto-optimal individuals in generation $t$) and $H(x,y)$ denotes the Hamming distance between $x$ and $y$. Since $\mathcal{F}_m \cap P_t \neq \emptyset$, we have $0 \leq d_t \leq n$ since the maximum possible Hamming distance between two Pareto-optimal search points is $n$. Note that we have created $y$ if $d_t=0$. 
    Since the population never loses all solutions with the same Pareto-optimal fitness vector (by Lemma~\ref{lem:Reference-Points}, 
    $d_t$ cannot increase. Further, for all $1 \leq k \leq n$, define the  random variable $X_k$ as the number of generations $t$ with $d_t=k$. Then the number of generations until there is a solution $y$ with $f(y)=v$ is at most $X=\sum_{k=1}^n X_k$. To decrease $d_t$, it suffices to choose an individual $x$ with $H(x,y)=d_t$ as a parent (which happens with probability at least $1/\mu$) and flip one specific bit in $x$ as follows. In a block $r$ with $f_{2r-1}(x) < f_{2r-1}(y)$ (and hence $f_{2r}(x)>f_{2r}(y)$) we may increase the number of leading ones and in a block $r$ with $f_{2r-1}(x)>f_{2r-1}(y)$ (and hence $f_{2r}(x) < f_{2r}(y)$) we may increase the number of trailing zeros. This generates a Pareto-optimal individual $z$ with $H(y,z)<H(x,y)$.
    Since a generation consists of $\mu$ trials, the probability for creating such a $y$ in one generation is at least 
    \[
    1-(1-s)^\mu \geq \frac{s\mu}{s\mu+1}:=p
    \]
    where $s$ is defined as in Phase~1. Hence, \neweditx{$X_k$ is stochastically dominated by a geometric random variable $Y_k$ with success probability $p$}, i.e. $E[Y_k] = 1/p = 1+en$. 
    Note that the $Y_k$ can be also seen as independent variables and as in the first phase we see that $P_t$ contains no solution $x$ with $f(x) = v$ after $3n^2$ generations with probability at most $e^{-\Omega(n)}$. 
    By a union bound over all possible~$v$, the probability is at most $(2n/m+1)^{m/2} \cdot e^{-\Omega(n)} = e^{-\Omega(n)}$ that the Pareto front is not completely covered after $3n^2$ generations. 
    As in the first phase, the expected number of trials required until the Pareto front is completely covered after $3\mu n^2$ evaluations is also $1+e^{-\Omega(n)}$. The bounds on the number of function evaluations follow by multiplying with $\mu$.
    \end{proof}

     A 
     \newedit{consequence} of Theorem~\ref{thm:Runtime-Analysis-NSGA-III-mLOTZ} is that \nsgaIII with a population size of $\mu=n+1$ can optimize the bi-objective $2$-\LOTZ function in $3n^2$ generations with overwhelming probability. The required population size is $n+1$ and coincides with the size of the Pareto front. In contrast, there are only positive results for \nsga in case that the population size is at least $4n+4$ (see~\cite{ZhengLuiDoerrAAAI22}). The reason is that \nsga uses a different mechanism for survival selection, the so called \emph{crowding distance}. We refer to~\cite{ZhengLuiDoerrAAAI22,ZhengD2023} for details.
     
\section{\mOMM}
\neweditx{We now extend} the result of Wietheger and Doerr from~\cite{WiethegerD23} on the expected 
\neweditx{runtime}
of \nsgaIII on $3$-\OMM to the \mOMM benchmark for arbitrary even values of~$m$. 
We \neweditx{also} divide $x$ in $m/2$ 
blocks, but 
\neweditx{we count both} the number\neweditx{s} of ones and zeros \neweditx{in each block}.

\begin{definition}[\citet{Zheng2023Inefficiency}]
Let $m$ be divisible by $2$ and let the problem size be a multiple of $m/2$. Then the $m$-objective function \mOMM is defined by
$\mOMM: \{0,1\}^n \to \mathbb{N}_0^m$ as 
\[
\mOMM(x) = (f_1(x), \ldots ,f_m(x))
\]
with 
\[
f_k(x)=
\begin{cases}
    \sum_{i=1}^{2n/m} x_{i+n(k-1)/m}, & \text{ if $k$ is odd,} \\
    \sum_{i=1}^{2n/m} (1-x_{i+n(k-2)/m}), & \text{ else,}
\end{cases}
\]
for all $x=(x_1, \ldots ,x_n) \in \{0,1\}^n$.
\end{definition}
Note that every search point is Pareto-optimal in the case of $\mOMM$. Since for each block there are only $2n/m+1$ distinct fitness values, a Pareto-optimal set has cardinality $(2n/m+1)^{m/2}$.

In the following, we say that an event occurs \emph{with high probability} if it occurs with probability at least $1-O(n^\varepsilon)$ for some constant~$\varepsilon > 0$.
\begin{theorem}
\label{thm:Runtime-Analysis-NSGA-III-mOMM}
Consider \nsgaIII with uniform selection
optimizing $f:=\mOMM$ for any \neweditx{constant} $m \in \mathbb{N}$ divisible by $2$ 
with \neweditx{$\varepsilon_{\text{nad}} \geq 2n/m$ and} a set $\refer$ of reference points as defined above for $p \in \mathbb{N}$ with $p \geq 4n\sqrt{m}$ and a population size $\mu \geq (2n/m+1)^{m/2}$. 
For every initial population, a Pareto-optimal set of $\mOMM$ is found in $O(n\ln n)$ generations and $O(\mu n \ln n)$ fitness evaluations with high probability and in expectation. 
\end{theorem}

\begin{proof}
    We upper bound the probability that after $cn\ln(n)$ generations, for some constant $c > 4m+4$, not the whole Pareto front is covered. 
    For this we first upper bound the probability that a solution $y$ with $f(y)=(v_1, \ldots ,v_m)$ where $v$ is a specific fitness vector with $v_{2i-1}+v_{2i} = 2n/m$ for $i \in [m]$ has not been generated after $cn\ln(n)$ generations. For each generation $t$ let $d_t:=\min_{x \in P_t}H(x,y)$. Note that $0 \leq d_t \leq n$ and that we have created an individual $y$ with $f(y)=v$ if $d_t=0$. 
    Since the population never loses all solutions with the same Pareto-optimal fitness vector (by Lemma~\ref{lem:Reference-Points}), $d_t$ cannot increase. As in the proof of Theorem~\ref{thm:Runtime-Analysis-NSGA-III-mLOTZ}, for all $1 \leq k \leq n$, define the random variable $X_k$ as the number of generations $t$ with $d_t=k$. Then the number of generations until there is a solution $y$ with $f(y)=v$ is at most $X=\sum_{k=1}^n X_k$. To decrease $k=d_t$, it suffices to choose an individual $x$ with $H(x,y)=k$ as a parent (which happens with probability at least $1/\mu$) and flip one of $k$ specific bits, while keeping the other bits unchanged. This happens with probability at least $k/(en)$.
    Since a generation consists of $\mu$ trials, the probability for decreasing $d_t$ in one generation is at least 
    \[
    1-(1-s_k)^\mu \geq \frac{s_k\mu}{s_k\mu+1}:=p_k
    \]
    where $s_k:=k/(e \mu n)$. Note that $p_k = k/(k+en) \geq k/(4n)$ and that \neweditx{$X_k$ is stochastically dominated by a geometric random variable $Y_k$ with success probability $p_k$} where the $Y_k$ are also independent. Let $Y:= \sum_{k=1}^n Y_k$. 
    By Theorem~16 in \cite{DOERR2019115} we have 
    \begin{align*}
    P(X \geq cn\ln(n)) &\leq P(Y \geq cn\ln(n)) \\
    &= P(Y \geq 4(1+\delta)n\ln(n)) \leq n^{-\delta}
    \end{align*}
    for $\delta:=c/4-1$, i.e. the probability that $R_t$ contains no solution $x$ with $f(x) = v$ after $cn\ln(n)$ generations is at most $n^{-\delta}$. By a union bound over all possible $v$, the probability that there is a fitness vector $v$ such that $P_t$ does not contain a Pareto-optimal solution $x$ with $f(x)=v$ after $cn\ln(n)$ generations 
    is at most $(2n/m+1)^{m/2} \cdot n^{-\delta}$. Using $m \ge 2$ and $\delta = c/4-1 > m$, this is at most $(n+1)^{m/2} \cdot n^{-m} = O(n^{-m/2})$. 
    Note that in $cn\ln(n)$ generations the algorithm makes $cn \ln(n)\mu$ fitness evaluations. The bound on the expected number of generations follows by repeating the above arguments with another period of $cn \ln(n)$ generations if necessary and noting that in expectation, $1+o(1)$ periods are sufficient. Multiplying by $\mu$ yields an upper bound on the expected number of evaluations.
    \end{proof}

Hence, \nsgaIII with a population size of $n+1$ (coinciding with the size of the Pareto front) is also able to optimize the 2-\OMM benchmark in polynomial time \neweditx{in expectation}. This cannot \neweditx{be} achieved with the \nsga: In~\citet{ZhengLuiDoerrAAAI22} \neweditx{it} is shown that the expected number of generations is $\exp(\Omega(n))$ until the population covers the whole Pareto front. 

\section{$m$-COCZ}

In \mCOCZ we also divide $x$ into two halves and the second half is further divided into $m/2$ many blocks of equal size $n/m$. In each block we maximize both the number of ones and the number of zeros, i.e. the objectives are \emph{conflicting} there. In the first part we only maximize the number of ones which goes equally into each objective. \newedit{We also see that a smaller population size is needed than in the $\mLOTZ$ or $\mOMM$-case.}
\begin{definition}[\citet{Laumanns2004}]
Let $m$ be divisible by $2$ and let the problem size be a multiple of $m$. Then the $m$-objective function \mCOCZ is defined by
$\mCOCZ: \{0,1\}^n \to \mathbb{N}_0^m$ as 
\[
\mCOCZ(x) = (f_1(x), \ldots ,f_m(x))
\]
with 
\[
f_k(x)= \sum_{i=1}^{n/2}x_i+
\begin{cases}
    \sum_{i=1}^{n/m} x_{i+n/2+(k-1)n/(2m)}, & \text{ if $k$ is odd,} \\
    \sum_{i=1}^{n/m} \neweditx{\left(1-x_{i+n/2+(k-2)n/(2m)}\right)}, & \text{ \neweditx{otherwise},}
\end{cases}
\]
for all $x=(x_1, \ldots ,x_n) \in \{0,1\}^n$.
\end{definition}

Note that a search point $x$ is Pareto-optimal with respect to \mCOCZ if and only if $x$ has $n/2$ many ones in the first half. Since for each block there are only $n/m+1$ distinct fitness values, a Pareto-optimal set has cardinality $(n/m+1)^{m/2}$. Hence, not every search point is Pareto-optimal with respect to \mCOCZ. However, in contrast to $\mLOTZ$ we can make a precise statement about the cardinality of a maximum set of mutually incomparable solutions: It coincides with the size of the Pareto front.

\begin{lemma}
\label{lem:Non-Dominated-Solutions-mCOTZ}
Let $\mathcal{F}$ be a Pareto \neweditx{optimal} set of $f:=\mCOCZ$ \neweditx{consisting of mutually incomparable solutions}. Then $\mathcal{F}$ is a set of mutually incomparable solutions with maximum cardinality for $f:=\mCOCZ$ with $\lvert{\mathcal{F}}\rvert=(n/m+1)^{m/2}$. 
\end{lemma}

\begin{proof}
    Let $S$ be a set of mutually incomparable solutions. We show that $\lvert{S}\rvert \leq (n/m+1)^{m/2}$. For a maximum set $M$ of solutions with distinct fitness vectors (i.e. $x,y \in M$ may be comparable, but $f(x) \neq f(y)$ if $x \neq y$) we have that $\lvert{M}\rvert = \lvert\{f(x) \mid x \in M\}\rvert = (n/2+1)(n/m+1)^{m/2}$ since there are $(n/2+1)(n/m+1)^{m/2}$ distinct fitness vectors: For a fixed number of ones in the first half of $x$ there are $(n/m+1)^{m/2}$ distinct solutions \neweditx{(which have all the same $f_1(x) + \ldots + f_m(x)$)} and two search points $x_1,x_2$ with a distinct number of ones in the first half have also distinct fitness vectors since $f_1(x_1) + \ldots + f_m(x_1) \neq f_1(x_2) + \ldots f_m(x_2)$.
    Now we partition $M$ into $(n/m+1)^{m/2}$ many sets (of equal cardinality) such that every two search points $x,y$ in the same set are comparable. Let
    \[
    W:=\{w \in \{0, \ldots ,n/m\}^m \mid w_{2i-1}+w_{2i}=n/m\}
    \]
    and for $w \in W$ let $M_w=\{x \in M \mid f(x) = w+\ell \cdot \vecone \text{ for } \ell \in \{0, \ldots , n/2\}\}$ where $\vecone:=(1, \ldots , 1) \in \mathbb{N}^m$, reflecting the possibility of having $\ell$ ones in the first half of the bit string that contribute to all objectives. Then $M= \bigcup_{w \in W} M_w$ and two solutions $x,y \in M_w$ are comparable: Let $x=w+\ell_1 \cdot \Vec{1}$ and $y=w+\ell_2 \cdot \Vec{1}$. \neweditx{Then $x \succeq y$ if $\ell_1 \leq \ell_2$ and $y \succeq x$ if $\ell_2 \leq \ell_1$}.  
    
    Hence, for every $w \in W$ there is at most one $x \in S$ with $x \in M_w$, which implies $\lvert{S}\rvert \leq \lvert{W}\rvert = (n/m+1)^{m/2}=\vert{\mathcal{F}}\vert$.
\end{proof}

Since $f_\max:=n/2+n/m$ for $f:=\mCOCZ$ we obtain the following result for the runtime of NSGA-III on $\mCOCZ$.

\begin{theoremrep}
\label{thm:Runtime-Analysis-NSGA-III-mOMM}
Consider \nsgaIII with uniform selection
optimizing $f:=\mCOCZ$ for any \neweditx{constant} $m \in \mathbb{N}$ with \neweditx{$\varepsilon_{\text{nad}} \geq n/2+n/m$ and} a set $\refer$ of reference points as defined above for $p \in \mathbb{N}$ with $p \geq 2m^{3/2}(n/2+n/m)$
\neweditx{and} a population size $\mu \geq (n/m+1)^{m/2}$. 
\newedit{For every initial population,} a Pareto-optimal set of $\mCOCZ$ is found in $O(n\ln n)$ generations and $O(\mu n \ln n)$ fitness evaluations with high probability and in expectation.
\end{theoremrep}


\begin{appendixproof}
    By Lemma~\ref{lem:fitnessvectors-non-dom-LOTZ} 
    the condition on the population size $\mu$ in Lemma~\ref{lem:Reference-Points} is always met. Along with $f_{\max}=n/2+n/m$ and $p \geq 2m^{3/2}(n/2+n/m) = 2m^{3/2}f_{\max}$, Lemma~\ref{lem:Reference-Points} is applicable in every generation $t$.
    
    We use the method of typical runs and divide the optimization procedure into two phases. We  
    show that with probability $1-o(1)$ each phase is completed in $cn/2\ln(n/2)$ generations, where \neweditx{$c >4m+8$} is a constant.

    \textbf{Phase $1$:} Create a Pareto-optimal search point.\\
    We upper bound the probability by $(n/2)^{-c/4-1}$ that a Pareto-optimal search point is not created in $cn/2\ln(n/2)$ generations. 
    Set $g_t:=\min_{x \in P_t}\sum_{i=1}^{n/2} (1-x_i) \in \{0\} \cup [n/2]$ (i.e. $g_t$ is the minimum number of zeros in the first half of an individual in the population). Then the population contains a Pareto-optimal solution if and only if $g_t=0$. Since a search point $x$ can only be dominated by a search point $y$ if $\sum_{i=1}^{n/2} x_i < \sum_{i=1}^{n/2} y_i$ (i.e. $\sum_{i=1}^{n/2} (1-y_i) < \sum_{i=1}^{n/2} (1-x_i)$), $g_t$ cannot increase by Lemma~\ref{lem:Reference-Points}.
    For $k \in [n/2]$ we define the random variable $X_k$ as the number of generations $t$ with $g_t=k$. Then the number of generations until there is a Pareto-optimal solution is \newedit{at most} $\sum_{k=1}^{n/2} X_k$. To decrease $g_t$, it suffices to choose an individual $y$ with $\sum_{i=1}^{n/2} (1-y_i)=k$ and to create an offspring $z$ with $\sum_{i=1}^{n/2} (1-z_i) < k$ while the second half of $z$ agrees with that of $y$. The former happens with probability at least $1/\mu$ and for the latter we may flip a zero in the first half of $y$ to one during mutation and keep the other bits unchanged, which happens with probability $k/n \cdot (1-1/n)^{n-1} \geq k/(en)$. Let $s_k:=k/(en \mu)$. Hence, in one generation (i.e. during $\mu$ offspring productions), the probability of decreasing $k$ is at least 
    \neweditx{
    \[
    1-(1-s_k)^{\mu} \geq \frac{s_k \mu}{s_k\mu + 1}:= \sigma_k.
    \]}
    Note that $\sigma_k = k/(k+en) \geq k/(4n) \neweditx{= (1/8)k/(n/2)}$ and hence, \neweditx{$X_k$ is stochastically dominated by a geometric random variable $Y_k$ with success probability~$\sigma_k$}. Let $Y:= \sum_{k=1}^{n/2} Y_k$. By Theorem~16 in \cite{DOERR2019115} 
    \neweditx{\begin{align*}
    P(X \geq cn/2\ln(n/2)) &\leq P(Y \geq cn/2\ln(n/2)) \\
    &= P(Y \geq 8(1+\delta)n/2 \ln(n/2)) \leq (n/2)^{-\delta}
    \end{align*}}
    for \neweditx{$\delta:=c/8-1$}, i.e. the probability that $P_t$ contains no solution $z$ with $z_i = 1$ for $i \in [n/2]$ is at most $(n/2)^{-\delta}$.
    To bound the expectation, we argue that in case the goal is not reached after $cn/2\ln(n/2)$ generations, we repeat the above arguments with another period of $cn/2\ln(n/2)$ generations. The expected number of periods required is at most $1 + o(1)$.
    

    \textbf{Phase $2$:} Cover the whole Pareto front.

    Note that, by Lemma~\ref{lem:Reference-Points}, Pareto-optimal fitness vectors can never disappear from the population.
    We upper bound the probability by \neweditx{$(n/2)^{-c/8+1+m/2} + (n/2)^{-c/8+1}$} that not the whole Pareto front is covered after $cn/2\ln(n/2)$ generations. For a specific fitness vector $v \in \mathbb{N}_0^m$ with $v_{2i-1}+v_{2i} = n/2+n/m$ we first upper bound the probability that a solution $\newedit{y}$ with \neweditx{$f(y)=v$} has not been created after $cn/2\ln(n/2)$ generations. Let $y \in \{0,1\}^n$ be with $f(y)=v$. We consider the distance of~$y$ to the closest Pareto-optimal search point in the population. For each generation $t$ let $d_t:=\min_{x \in \mathcal{F}_m \cap P_t}H(x,y)$ where $\mathcal{F}_m$ denotes the set of all Pareto-optimal search points (i.e. $\mathcal{F}_m \cap P_t$ is the set of all Pareto-optimal individuals in generation $t$) and $H(x,y)$ denotes the Hamming distance between $x$ and $y$. Since $\mathcal{F}_m \cap P_t \neq \emptyset$, we have $0 \leq d_t \leq n/2$ since the maximum possible Hamming distance between two Pareto-optimal search points is $n/2$. Note that we have created $y$ if $d_t=0$. 
    Since the population never loses all solutions with the same Pareto-optimal fitness vector (by Lemma~\ref{lem:Reference-Points}), $d_t$ cannot increase. Further, for all $1 \leq k \leq n/2$, define the random variable $X_k$ as the number of generations $t$ with $d_t=k$. Then the number of generations until there is a solution $y$ with $f(y)=v$ is at most $X=\sum_{k=1}^{n/2} X_k$. To decrease $d_t$, it suffices to choose an individual $x$ with ${H(x,y)=d_t=k}$ as a parent (which happens with probability at least $1/\mu$) and flip one of $k$ specific bits while keeping the remaining bits unchanged, which happens with probability at least $k/(en)$.
    Since a generation consists of $\mu$ trials, the probability for creating such a $y$ in one generation is at least 
    \[
    1-(1-s_k)^\mu \geq \frac{s_k\mu}{s_k\mu+1}=\sigma_k
    \]
    where $s_k:=k/(e\mu n)$. Hence, \neweditx{$X_k$ is stochastically dominated by a geometric random variable $Y_k$ with success probability $\sigma_k$}. \neweditx{As above we have $\sigma_k \geq (1/8)k/(n/2)$}. Let $Y:=\sum_{k=1}^{n/2} Y_k$. Again by Theorem~16 in \cite{DOERR2019115}
    \neweditx{
    \begin{align*}
    P(X \geq cn/2\ln(n/2)) &\leq P(Y \geq cn/2\ln(n/2)) \\
    &= P(Y \geq 8(1+\delta)n/2 \ln(n/2)) \leq (n/2)^{-\delta}
    \end{align*}}
    for \neweditx{$\delta:=c/8-1$}, i.e. the probability that $P_t$ contains no solution $y$ with $f(y)=v$ is at most $(n/2)^{-\delta}$. By a union bound over all possible $v$, the probability is at most 
    \neweditx{
    \begin{align*}
    (n/m+1)^{m/2} \cdot (n/2)^{-\delta} &\leq (n/2)^{-\delta + m/2} + (n/2)^{-\delta} \\
    &= (n/2)^{-c/8+1+m/2} + (n/2)^{-c/8+1} = o(1)
    \end{align*}}
    that the Pareto front is not completely covered after $cn/2\ln(n/2)$ generations \neweditx{since $c > 4m+8$}. We obtain a bound on the expected number of generations by repeating the above arguments if necessary with another period of $cn/2\ln(n/2)$ generations and noting that in expectation at most $1+o(1)$ periods are sufficient. The bound on the number of function evaluations \newedit{follows} by multiplying with $\mu$.
\end{appendixproof}




\section{Conclusions}

We have introduced new mathematical tools and techniques for the runtime analysis
of NSGA-III on many-objective optimization problems. These techniques were 
demonstrated on the standard benchmark functions \mLOTZ, \mOMM and \mCOCZ to show
how the upper bounds on the expected running time of the algorithms scale \neweditx{with the problem dimension for a constant number of objectives}. Our methods also provide ways
to parameterize \neweditx{NSGA-III} so that these bounds on the performance can be guaranteed. 
To our knowledge, this is the first runtime analysis of NSGA-III for more than
three objectives. We hope that the techniques developed in this paper will serve as \newedit{ 
a stepping stone} for \neweditx{future analyses} of NSGA-III and 
will provide a better understanding of the capabilities of this algorithm. 
As future work, we plan to investigate the performance of NSGA-III on combinatorial 
multi-objective problems, such as computing approximation guarantees for multiobjective minimum spanning trees \cite{DBLP:journals/eor/Neumann07}. 

\ifreview\else
\ifarxiv
\section*{Acknowledgement} 
\else
\begin{acks}
\fi
Frank Neumann has been supported by the Australian Research Council through grant FT200100536.
\ifarxiv\else
\end{acks}
\fi
\fi

\appendix

\section{Useful results}

\begin{lemma}
\label{lem:Trigonometry}
    Let $x \in (0,1]$. Then $2\arcsin(x/2) < \arcsin(x)$.
\end{lemma}
\begin{proof}
For $a,b \in \mathbb{R}$ we use the well known trigonometric identity $\sin(a+b)=\sin(a)\cos(b)+\cos(a)\sin(b)$.
For $y,z \in [-1,1]$ we obtain by using $\cos(y)=\sqrt{1-\sin^2(y)}$
\begin{align*}
& \sin(\arcsin(y)+\arcsin(z)) \\
&=\sin(\arcsin(y))\cos(\arcsin(z))+\cos(\arcsin(y))\sin(\arcsin(z))\\
&=y\sqrt{1-z^2}+z\sqrt{1-y^2}.
\end{align*}
And hence if additionally $y,z \in [-0.5,0.5]$ (since then $-\pi/2 \leq \arcsin(y)+\arcsin(z) \leq \pi/2$) 
\[
\arcsin(y) + \arcsin(z) = \arcsin(y \cdot \sqrt{1-z^2} + z \cdot \sqrt{1-y^2})
\]
which implies for $x \in (0,1]$ (where $y=z=x/2$)
\begin{align*}
2\arcsin(x/2) &= \arcsin((x/2) \sqrt{1-x^2/4}+(x/2) \sqrt{1-x^2/4})\\ 
&= \arcsin(x \sqrt{1-x^2/4}) < \arcsin(x)
\end{align*}
where the inequality holds since $\arcsin$ \neweditx{strictly increases} on $\neweditx{(}0,1]$. 
\end{proof}

\balance
\bibliographystyle{abbrvnat}
\bibliography{references}

\end{document}

%% file: preamble.tex
\usepackage[utf8]{inputenc}

\ifarxiv
\usepackage[a4paper,left=2cm,right=2cm,top=2cm,bottom=2cm]{geometry}
\usepackage[onehalfspacing]{setspace}
\usepackage[numbers,longnamesfirst]{natbib}
\usepackage{amsthm,amsthm,amssymb}
\fi

\usepackage{amsmath,mathtools} 
\usepackage{dsfont}

\usepackage{booktabs} 
\usepackage{xspace}
\usepackage{graphicx}
\usepackage{xspace}
\usepackage{url}
\usepackage{enumerate}
\usepackage{breqn}

\usepackage{caption} 
\usepackage{subcaption}

\usepackage[algo2e,ruled,vlined,linesnumbered]{algorithm2e} 
\SetAlgoSkip{}

\ifarxiv
\usepackage[appendix=inline]{apxproof} 
\fi

\ifreview

\fi

\usepackage{balance} 

\allowdisplaybreaks[3] 

\usepackage{tikz}
\usetikzlibrary{shapes,arrows}
\usetikzlibrary{decorations}
\usetikzlibrary{plotmarks}
\usetikzlibrary{calc}
\usetikzlibrary{mindmap}
\usetikzlibrary{shadows}
\usetikzlibrary{backgrounds}
\usetikzlibrary{patterns}
\usetikzlibrary{shapes.symbols}

\usepackage{pgfplots}
\usepackage{pgfplotstable}
\usepgfplotslibrary{statistics}

\ifarxiv
\newtheorem{theorem}             {Theorem}
\newtheorem{lemma}      [theorem]{Lemma}

\newtheorem{definition} [theorem]{Definition}

\newtheorem{example}[theorem]{Example}
\newtheoremrep{theorem}{Theorem}[section]
\newtheoremrep{lemma}[theorem]{Lemma}
\newtheoremrep{corollary}[theorem]{Corollary}

\fi

\usepackage{tabularx,colortbl}

\makeatletter
\g@addto@macro\bfseries{\boldmath}
\makeatother

\newcommand{\ie}{i.\,e.\xspace}
\newcommand{\cf}{cf.\xspace}

\newcommand{\R}{\mathbb{R}}

\DeclareMathOperator{\Unif}{Unif}                         

\newcommand{\vecone}{\vec{1}}

\newcommand{\TZ}{\textsc{TZ}\xspace}                      
\newcommand{\LO}{\textsc{LO}\xspace}                      

\newcommand{\OMM}{\textsc{OMM}\xspace}                             

\newcommand{\LOTZ}{\textsc{LOTZ}\xspace}
\newcommand{\mLOTZ}{\ensuremath{m}\text{-}\textsc{LOTZ}\xspace}
\newcommand{\mOMM}{\ensuremath{m}\text{-}\textsc{OMM}\xspace}
\newcommand{\mCOCZ}{\ensuremath{m}\text{-}\textsc{COCZ}\xspace}

\newcommand{\mLOTZFULL}{\ensuremath{m}\text{-}\textsc{LeadingOnesTrailingZeroes}\xspace}
\newcommand{\mOMMFULL}{\ensuremath{m}\text{-}\textsc{OneMinMax}\xspace}
\newcommand{\mCOCZFULL}{\ensuremath{m}\text{-}\textsc{CountingOnesCountingZeros}\xspace}

\newcommand{\rp}{\mathrm{rp}}
\newcommand{\refer}{\mathcal{R}_p}

\newcommand{\nsga}{NSGA\nobreakdash-II\xspace}
\newcommand{\nsgaIII}{NSGA\nobreakdash-III\xspace}

\newcommand{\toggleplot}[1]{{\textcolor{red}{Plots removed to increase compilation speed. Use command $\backslash$toggleplot in preamble to reinsert them.}}}

\iffinal
\newcommand{\andre}[1]{\textcolor{orange}{}}
\newcommand{\dirk}[1]{\textcolor{brown}{}}
\newcommand{\cuong}[1]{\textcolor{cyan}{}}
\newcommand{\frank}[1]{\textcolor{magenta}{}}
\newcommand{\newedit}[1]{\textcolor{black}{#1}}
\newcommand*{\todo}[1]{\textcolor{red}{}}
\newcommand{\neweditx}[1]{\textcolor{black}{#1}} 
\else
\newcommand{\andre}[1]{\textcolor{orange}{[(Andre) #1]}}
\newcommand{\dirk}[1]{\textcolor{brown}{[(Dirk) #1]}}
\newcommand{\cuong}[1]{\textcolor{cyan}{[(Cuong) #1]}}
\newcommand{\frank}[1]{\textcolor{magenta}{[(Frank) #1]}}
\newcommand{\newedit}[1]{\textcolor{black}{#1}}
\newcommand*{\todo}[1]{\textcolor{red}{\textrm{(TODO: #1)}}}
\newcommand{\neweditx}[1]{\textcolor{blue}{#1}}
\fi

%% file: NSGA-III_GECCO.bbl
\begin{thebibliography}{41}
\providecommand{\natexlab}[1]{#1}
\providecommand{\url}[1]{\texttt{#1}}
\expandafter\ifx\csname urlstyle\endcsname\relax
  \providecommand{\doi}[1]{doi: #1}\else
  \providecommand{\doi}{doi: \begingroup \urlstyle{rm}\Url}\fi

\bibitem[Badkobeh et~al.(2015)Badkobeh, Lehre, and Sudholt]{Badkobeh2015}
G.~Badkobeh, P.~K. Lehre, and D.~Sudholt.
\newblock Black-box complexity of parallel search with distributed populations.
\newblock In \emph{Proceedings of the Foundations of Genetic Algorithms
  ({FOGA}'15)}, pages 3--15. {ACM} Press, 2015.

\bibitem[Bian and Qian(2022)]{Bian2022PPSN}
C.~Bian and C.~Qian.
\newblock Better running time of the non-dominated sorting genetic algorithm
  {II} ({NSGA-II}) by using stochastic tournament selection.
\newblock In \emph{Proceedings of the International Conference on Parallel
  Problem Solving from Nature (PPSN~'22)}, volume 13399 of \emph{LNCS}, pages
  428--441. Springer, 2022.

\bibitem[Blank et~al.(2019)Blank, Deb, and Roy]{Blank2019}
J.~Blank, K.~Deb, and P.~C. Roy.
\newblock Investigating the normalization procedure of {NSGA-III}.
\newblock In K.~Deb, E.~Goodman, C.~A. Coello~Coello, K.~Klamroth,
  K.~Miettinen, S.~Mostaghim, and P.~Reed, editors, \emph{Evolutionary
  Multi-Criterion Optimization}, pages 229--240, Cham, 2019. Springer
  International Publishing.

\bibitem[Cerf et~al.(2023)Cerf, Doerr, Hebras, Kahane, and
  Wietheger]{DBLP:conf/ijcai/CerfDHKW23}
S.~Cerf, B.~Doerr, B.~Hebras, Y.~Kahane, and S.~Wietheger.
\newblock The first proven performance guarantees for the non-dominated sorting
  genetic algorithm {II} {(NSGA-II)} on a combinatorial optimization problem.
\newblock In \emph{Proceedings of the International Joint Conference on
  Artificial Intelligence, {IJCAI}~2023}, pages 5522--5530. ijcai.org, 2023.

\bibitem[Coello et~al.(2013)Coello, Van~Veldhuizen, and
  Lamont]{coello2013evolutionary}
C.~Coello, D.~Van~Veldhuizen, and G.~Lamont.
\newblock \emph{Evolutionary Algorithms for Solving Multi-Objective Problems}.
\newblock Genetic Algorithms and Evolutionary Computation. Springer US, 2013.

\bibitem[Dang et~al.(2023{\natexlab{a}})Dang, Opris, Salehi, and
  Sudholt]{DaOp2023}
D.-C. Dang, A.~Opris, B.~Salehi, and D.~Sudholt.
\newblock Analysing the robustness of {NSGA-II} under noise.
\newblock In \emph{Proceedings of the Genetic and Evolutionary Computation
  Conference ({GECCO}'23)}, pages 642--651. {ACM} Press, 2023{\natexlab{a}}.

\bibitem[Dang et~al.(2023{\natexlab{b}})Dang, Opris, Salehi, and
  Sudholt]{Dang2023}
D.-C. Dang, A.~Opris, B.~Salehi, and D.~Sudholt.
\newblock A proof that using crossover can guarantee exponential speed-ups in
  evolutionary multi-objective optimisation.
\newblock In \emph{Proceedings of the {AAAI} Conference on Artificial
  Intelligence, {AAAI}~2023}, pages 12390--12398. {AAAI} Press,
  2023{\natexlab{b}}.

\bibitem[Dang et~al.(2024)Dang, Opris, and Sudholt]{Dang2024}
D.-C. Dang, A.~Opris, and D.~Sudholt.
\newblock Crossover can guarantee exponential speed-ups in evolutionary
  multi-objective optimisation.
\newblock \emph{Artificial Intelligence}, 330:\penalty0 104098, 2024.
\newblock ISSN 0004-3702.

\bibitem[Das and Dennis(1998)]{Das1998}
I.~Das and J.~E. Dennis.
\newblock Normal-boundary intersection: {A} new method for generating the
  pareto surface in nonlinear multicriteria optimization problems.
\newblock \emph{{SIAM} Journal on Optimization}, 8\penalty0 (3):\penalty0
  631--657, 1998.

\bibitem[Deb(2001)]{kdeb01}
K.~Deb.
\newblock \emph{Multi-Objective Optimization using Evolutionary Algorithms}.
\newblock John Wiley \& Sons, 2001.

\bibitem[Deb and Jain(2014)]{DebJain2014}
K.~Deb and H.~Jain.
\newblock An evolutionary many-objective optimization algorithm using
  reference-point-based nondominated sorting approach, part i: Solving problems
  with box constraints.
\newblock \emph{IEEE Transactions on Evolutionary Computation}, 18\penalty0
  (4):\penalty0 577--601, 2014.

\bibitem[Deb et~al.(2002)Deb, Pratap, Agarwal, and Meyarivan]{Deb2002}
K.~Deb, A.~Pratap, S.~Agarwal, and T.~Meyarivan.
\newblock A fast and elitist multiobjective genetic algorithm: {NSGA-II}.
\newblock \emph{{IEEE} Transactions on Evolutionary Computation}, 6\penalty0
  (2):\penalty0 182--197, 2002.

\bibitem[Do et~al.(2023)Do, Neumann, Neumann, and Sutton]{do2023rigorous}
A.~V. Do, A.~Neumann, F.~Neumann, and A.~M. Sutton.
\newblock Rigorous runtime analysis of {MOEA/D} for solving multi-objective
  minimum weight base problems.
\newblock In \emph{Proceedings of the Annual Conference on Neural Information
  Processing Systems, {NeurIPS}~2023}, 2023.

\bibitem[Doerr(2019)]{DOERR2019115}
B.~Doerr.
\newblock Analyzing randomized search heuristics via stochastic domination.
\newblock \emph{Theoretical Computer Science}, 773:\penalty0 115--137, 2019.
\newblock ISSN 0304-3975.

\bibitem[Doerr and Neumann(2020)]{DoerrN20}
B.~Doerr and F.~Neumann, editors.
\newblock \emph{Theory of Evolutionary Computation -- Recent Developments in
  Discrete Optimization}.
\newblock Springer, 2020.

\bibitem[Doerr and Qu(2022)]{DoerrQu22}
B.~Doerr and Z.~Qu.
\newblock A first runtime analysis of the {NSGA-II} on a multimodal problem.
\newblock In \emph{Proceedings of the International Conference on Parallel
  Problem Solving from Nature ({PPSN}~'22)}, volume 13399 of \emph{LNCS}, pages
  399--412. Springer, 2022.

\bibitem[Doerr and Qu(2023)]{Doerr2023}
B.~Doerr and Z.~Qu.
\newblock A first runtime analysis of the {NSGA}-{II} on a multimodal problem.
\newblock \emph{{IEEE} Transactions on Evolutionary Computation}, pages 1--1,
  2023.

\bibitem[Feller(1968)]{Feller1968}
W.~Feller.
\newblock \emph{An introduction to probability theory and its applications.
  {Volume} {I}}.
\newblock John Wiley \& Sons, 1968.

\bibitem[Friedrich et~al.(2010)Friedrich, He, Hebbinghaus, Neumann, and
  Witt]{DBLP:journals/ec/FriedrichHHNW10}
T.~Friedrich, J.~He, N.~Hebbinghaus, F.~Neumann, and C.~Witt.
\newblock Approximating covering problems by randomized search heuristics using
  multi-objective models.
\newblock \emph{Evolutionary Computation}, 18\penalty0 (4):\penalty0 617--633,
  2010.

\bibitem[Horoba(2009)]{DBLP:conf/foga/Horoba09}
C.~Horoba.
\newblock Analysis of a simple evolutionary algorithm for the multiobjective
  shortest path problem.
\newblock In \emph{Proceedings of the Foundations of Genetic Algorithms
  ({FOGA}'09)}, pages 113--120. {ACM} Press, 2009.

\bibitem[Horoba and Neumann(2008)]{DBLP:conf/gecco/HorobaN08}
C.~Horoba and F.~Neumann.
\newblock Benefits and drawbacks for the use of epsilon-dominance in
  evolutionary multi-objective optimization.
\newblock In \emph{Proceedings of the Genetic and Evolutionary Computation
  Conference ({GECCO}'08)}, pages 641--648. {ACM} Press, 2008.

\bibitem[Horoba and Neumann(2009)]{DBLP:conf/foga/HorobaN09}
C.~Horoba and F.~Neumann.
\newblock Additive approximations of pareto-optimal sets by evolutionary
  multi-objective algorithms.
\newblock In \emph{Proceedings of the Foundations of Genetic Algorithms
  ({FOGA}'09)}, pages 79--86. {ACM} Press, 2009.

\bibitem[Horoba and Neumann(2010)]{DBLP:series/sci/HorobaN10}
C.~Horoba and F.~Neumann.
\newblock Approximating pareto-optimal sets using diversity strategies in
  evolutionary multi-objective optimization.
\newblock In \emph{Advances in Multi-Objective Nature Inspired Computing},
  volume 272 of \emph{Studies in Computational Intelligence}, pages 23--44.
  Springer, 2010.

\bibitem[Huang et~al.(2019)Huang, Zhou, Chen, and He]{Huang2019}
Z.~Huang, Y.~Zhou, Z.~Chen, and X.~He.
\newblock Running time analysis of {MOEA/D} with crossover on discrete
  optimization problem.
\newblock In \emph{Proceedings of the {AAAI} Conference on Artificial
  Intelligence, {AAAI}~2019}, pages 2296--2303. {AAAI} Press, 2019.

\bibitem[Huang et~al.(2021{\natexlab{a}})Huang, Zhou, Chen, He, Lai, and
  Xia]{Huang20211}
Z.~Huang, Y.~Zhou, Z.~Chen, X.~He, X.~Lai, and X.~Xia.
\newblock Running time analysis of {MOEA/D} on pseudo-boolean functions.
\newblock \emph{{IEEE} Transactions on Cybernetics}, 51\penalty0 (10):\penalty0
  5130--5141, 2021{\natexlab{a}}.

\bibitem[Huang et~al.(2021{\natexlab{b}})Huang, Zhou, Luo, and Lin]{Huang2021}
Z.~Huang, Y.~Zhou, C.~Luo, and Q.~Lin.
\newblock A runtime analysis of typical decomposition approaches in {MOEA/D}
  framework for many-objective optimization problems.
\newblock In \emph{Proceedings of the International Joint Conference on
  Artificial Intelligence, {IJCAI}~2021}, pages 1682--1688. ijcai.org,
  2021{\natexlab{b}}.

\bibitem[Jansen(2013)]{Jansen2013}
T.~Jansen.
\newblock \emph{Analyzing Evolutionary Algorithms -- The Computer Science
  Perspective}.
\newblock Natural Computing Series. Springer, 2013.

\bibitem[Knowles et~al.(2001)Knowles, Watson, and
  Corne]{DBLP:conf/emo/KnowlesWC01}
J.~D. Knowles, R.~A. Watson, and D.~Corne.
\newblock Reducing local optima in single-objective problems by
  multi-objectivization.
\newblock In \emph{Proceedings of the International Conference on Evolutionary
  Multi-Criterion Optimization, {EMO}~2001}, volume 1993 of \emph{Lecture Notes
  in Computer Science}, pages 269--283. Springer, 2001.

\bibitem[Laumanns et~al.(2004)Laumanns, Thiele, and Zitzler]{Laumanns2004}
M.~Laumanns, L.~Thiele, and E.~Zitzler.
\newblock Running time analysis of multiobjective evolutionary algorithms on
  pseudo-boolean functions.
\newblock \emph{{IEEE} Transactions on Evolutionary Computation}, 8\penalty0
  (2):\penalty0 170--182, 2004.

\bibitem[Neumann(2007)]{DBLP:journals/eor/Neumann07}
F.~Neumann.
\newblock Expected runtimes of a simple evolutionary algorithm for the
  multi-objective minimum spanning tree problem.
\newblock \emph{European Journal of Operational Research}, 181\penalty0
  (3):\penalty0 1620--1629, 2007.

\bibitem[Neumann and Wegener(2006)]{DBLP:journals/nc/NeumannW06}
F.~Neumann and I.~Wegener.
\newblock Minimum spanning trees made easier via multi-objective optimization.
\newblock \emph{Natural Computing}, 5\penalty0 (3):\penalty0 305--319, 2006.

\bibitem[Neumann and Witt(2010)]{NeumannWittBook}
F.~Neumann and C.~Witt.
\newblock \emph{Bioinspired Computation in Combinatorial
  Optimization~--~Algorithms and Their Computational Complexity}.
\newblock Natural Computing Series. Springer, 2010.

\bibitem[Wietheger and Doerr(2023)]{WiethegerD23}
S.~Wietheger and B.~Doerr.
\newblock A mathematical runtime analysis of the non-dominated sorting genetic
  algorithm {III} {(NSGA-III)}.
\newblock In \emph{Proceedings of the International Joint Conference on
  Artificial Intelligence, {IJCAI}~2023}, pages 5657--5665. ijcai.org, 2023.

\bibitem[Wietheger and Doerr(2024)]{WiethegerD24}
S.~Wietheger and B.~Doerr.
\newblock Near-tight runtime guarantees for many-objective evolutionary
  algorithms.
\newblock Under review, 2024.

\bibitem[Witt(2014)]{WITT201438}
C.~Witt.
\newblock Fitness levels with tail bounds for the analysis of randomized search
  heuristics.
\newblock \emph{Information Processing Letters}, 114\penalty0 (1):\penalty0
  38--41, 2014.

\bibitem[Zheng and Doerr(2022)]{Zheng2022}
W.~Zheng and B.~Doerr.
\newblock Better approximation guarantees for the {NSGA-II} by using the
  current crowding distance.
\newblock In \emph{Proceedings of the Genetic and Evolutionary Computation
  Conference ({GECCO}'22)}, pages 611--619. {ACM} Press, 2022.

\bibitem[Zheng and Doerr(2023{\natexlab{a}})]{Zheng2023Inefficiency}
W.~Zheng and B.~Doerr.
\newblock Runtime analysis for the {NSGA-II}: Proving, quantifying, and
  explaining the inefficiency for many objectives.
\newblock \emph{IEEE Transactions on Evolutionary Computation},
  2023{\natexlab{a}}.
\newblock To appear.

\bibitem[Zheng and Doerr(2023{\natexlab{b}})]{ZhengD2023}
W.~Zheng and B.~Doerr.
\newblock Mathematical runtime analysis for the non-dominated sorting genetic
  algorithm {II} {(NSGA-II)}.
\newblock \emph{Artificial Intelligence}, 325:\penalty0 104016,
  2023{\natexlab{b}}.

\bibitem[Zheng and Doerr(2024)]{DBLP:conf/aaai/0001D24}
W.~Zheng and B.~Doerr.
\newblock Runtime analysis of the {SMS-EMOA} for many-objective optimization.
\newblock In \emph{Proceedings of the {AAAI} Conference on Artificial
  Intelligence, {AAAI}~2024}, pages 20874--20882. {AAAI} Press, 2024.

\bibitem[Zheng et~al.(2022)Zheng, Liu, and Doerr]{ZhengLuiDoerrAAAI22}
W.~Zheng, Y.~Liu, and B.~Doerr.
\newblock A first mathematical runtime analysis of the non-dominated sorting
  genetic algorithm {II} {(NSGA-II)}.
\newblock In \emph{Proceedings of the {AAAI} Conference on Artificial
  Intelligence, {AAAI}~2022}, pages 10408--10416. {AAAI} Press, 2022.

\bibitem[Zhou et~al.(2019)Zhou, Yu, and Qian]{DBLP:books/sp/ZhouYQ19}
Z.~Zhou, Y.~Yu, and C.~Qian.
\newblock \emph{Evolutionary learning: Advances in theories and algorithms}.
\newblock Springer, 2019.

\end{thebibliography}
